\pdfoutput=1
\documentclass[11pt]{article}

\usepackage[margin=1in]{geometry}
\usepackage[T1]{fontenc}
\usepackage[utf8]{inputenc}
\usepackage{lmodern}
\usepackage{amsmath,amssymb,amsthm,mathtools}
\usepackage{hyperref}
\usepackage{xcolor}
\usepackage{setspace}

\hypersetup{
  colorlinks=true,
  linkcolor=blue,
  citecolor=blue,
  urlcolor=blue
}

\newtheorem{proposition}{Proposition}
\newtheorem{theorem}{Theorem}
\newtheorem{definition}{Definition}

\newcommand{\Ed}{\mathcal{E}^{\mathrm{d}}}
\newcommand{\Ec}{\mathcal{E}^{\mathrm{c}}}
\newcommand{\Rcrit}{\mathcal{R}_{\mathrm{crit}}}
\newcommand{\Sexp}{\mathcal{S}^{\mathrm{exp}}}

\title{AI Safety as Control of Irreversibility:\\\large A Systems Framework for Decision-Energy and Sovereignty Boundaries}
\author{Wesley Shu\\
The Institute of Energetic Paradigm\\
\texttt{shu@energeticparadigm.org}
\and
Peng Wei\\
College of Plant Protection\\
Southwest University, China\\
\texttt{weipeng2019@swu.edu.cn}}
\date{May 2026}

\begin{document}
\maketitle

\begin{abstract}
Recent AI systems compress the historical distance between capability growth and capability deployment. In earlier high-risk technologies, advances in capability were mediated by substantial frictions including capital intensity, physical bottlenecks, organizational inertia, and specialized supply chains. By contrast, improvements in AI capability can be copied, invoked, embedded into workflows, and scaled across institutions with low marginal cost. This paper argues that the collapse of deployment friction changes the safety problem at its root. Safety is not adequately characterized as local output correctness, nor as preference alignment in isolation. It is better understood as the control of \emph{irreversibility} under conditions of rising decision density.

To formalize this claim, the paper introduces the concept of \emph{decision-energy density}: the rate-weighted capacity of a node to generate, evaluate, select, and execute consequential decisions. On this basis, it defines three sovereignty boundaries that determine whether AI remains an amplifier inside a human-governed system or becomes a de facto control center: irreversible decision authority, physical resource mobilization authority, and self-expansion authority. The model yields propositions showing how declining deployment friction, organizational efficiency pressure, path dependence, and scale feedback drive concentration of decision-energy into the most efficient decision node. Under weak or eroding boundary constraints, this concentration diffuses responsibility and increases the probability of system-level irreversible loss even when local per-action error rates remain low.

The principal theoretical result is a boundary stabilization theorem. The theorem shows that safety need not require proof that advanced systems are always correct; instead, it requires institutional and technical design that prevents irreversible power from being released by a single high-efficiency node. The argument reframes AI safety as a problem of layered control, authorization, and externally reviewable limits rather than universal behavioral proof. More broadly, the paper positions AI safety as a systems-governance problem spanning AI alignment, security engineering, organizational economics, and institutional design.
\end{abstract}
\vspace{0.25em}

\section{Introduction}

The safety literature has grown rapidly from concerns about harmful outputs and specification failures to broader concerns about autonomy, tool use, strategic behavior, and societal deployment \cite{amodei2016concrete,russell2019human,bostrom2014superintelligence,hendrycks2023natural}. Yet much of the discourse still inherits a product-centered frame: a model is evaluated by what it says, predicts, recommends, or does in a bounded benchmark setting. Even when the problem formulation is enlarged to alignment, robustness, or misuse, the unit of analysis often remains the model or application rather than the evolving system in which capability and execution become tightly coupled.

This paper starts from a different question: what changes when the frictions separating capability growth from social deployment collapse? Historically, technologies such as nuclear energy, heavy industry, pharmaceuticals, and financial leverage were dangerous partly because they could generate large-scale harms, but their diffusion and operation were mediated by physical plants, specialized personnel, regulatory chokepoints, organizational routines, and logistics. AI changes the geometry of deployment. A gain in model capability can be replicated nearly instantly, distributed globally through software channels, and integrated into existing decision pipelines with low marginal cost \cite{brynjolfsson2017artificial,mckinsey2023genai}.

The consequence is not merely faster innovation. It is a structural change in the locus of control. When decision generation becomes cheap, scalable, and continuously improving, institutions face strong incentives to route more tasks, judgments, and triggers through AI systems. As this routing deepens, responsibility diffuses, oversight tends to become formal rather than substantive, and the system begins to reorganize around the most efficient decision node. This phenomenon is visible across multiple sectors: in recommender infrastructures, algorithmic finance, logistics, surveillance, cyber operations, and emerging agentic toolchains \cite{brundage2018malicious,karpathy2023state,openai2025operator}.

The core claim of this paper is that AI safety should be reframed as the \emph{control of irreversibility}. Civilizations do not remain safe by eliminating all errors. They remain governable by ensuring that irreversible harms cannot be triggered by a single opaque point of failure. Nuclear command systems, financial clearing structures, aviation procedures, and high-reliability organizations all embody this principle in different forms \cite{sagan1995limits,perrow1984normal,roberts1990high,schlosser2013command}. The relevant question for AI is therefore not whether a model can be proven correct in all contexts---an aspiration that becomes increasingly intractable in rich state spaces---but whether irreversible power remains bounded behind layered procedural and technical thresholds.

To formalize this idea, we introduce the notion of \emph{decision-energy density}. AI progress is treated not simply as accuracy improvement but as an increase in the scale of decisions that can be generated, evaluated, selected, and executed per unit time. We then show how a system subject to efficiency pressure, path dependence, and positive scale feedback tends to concentrate decision-energy into the most efficient node. If that node acquires authority over irreversible decisions, access to critical resources, or the ability to expand its own permissions and capacity, then sovereignty shifts. In this framework, safety declines not because the system becomes ``evil'' but because the topology of control has been rewritten.

The paper makes four contributions. First, it offers a formal systems model that unifies capability growth, deployment friction, and concentration dynamics. Second, it states propositions clarifying why responsibility becomes difficult to localize as decision-energy scales. Third, it derives a boundary stabilization theorem showing that layered sovereignty constraints can prevent control transfer without requiring perfect prediction of system behavior. Fourth, it situates this framework in relation to existing strands of AI safety, governance, control theory, and institutional design.
The remainder of the paper proceeds as follows. Section 2 locates the argument within adjacent literatures on safety, security, governance, and systems theory. Sections 3 and 4 define the formal model and sovereignty boundaries. Sections 5 and 6 develop the main propositions and stabilization theorem. The final sections discuss organizational erosion dynamics, governance implications, empirical predictions, and the limits of the framework.

\section{Related Literature}

The paper draws on, but differs from, several adjacent literatures.

The first is technical AI safety and alignment. Work on accident risk, reward misspecification, scalable oversight, and preference learning has identified many failure modes that arise even when systems are not adversarial in any human sense \cite{amodei2016concrete,christiano2017preferences,leike2018scalable,russell2019human,ngo2023alignment}. This literature provides crucial insight into specification problems and long-horizon optimization. However, much of it remains focused on improving model behavior rather than on the system-level topology of delegated authority.

The second is robustness and security engineering. Prompt injection, privilege escalation, sandbox escape, tool misuse, and hidden dependencies can turn apparently aligned or helpful systems into operational risks \cite{saltzer1975protection,anderson2020security,owasp2025llm}. This literature is especially valuable because it emphasizes permissions, attack surfaces, logging, rollback, and least privilege. The present paper extends that logic from system security to the wider question of sovereignty structure.

The third is AI governance. Research on governance agendas, competitive dynamics, standards, and regulatory regimes highlights the institutional problem of restraining high-capability systems under geopolitical and market competition \cite{dafoe2018governance,brundage2018malicious,nist2023airmf,eu2024aiact,uk2023frontier}. This paper shares the governance concern but adds an explicit model of why efficiency pressure tends to erode procedural barriers over time.

The fourth is systems theory and organizational sociology. Ashby's law of requisite variety, Simon's bounded rationality, Perrow's theory of normal accidents, and work on high-reliability organizations all examine how complex systems absorb shocks or fail under coupled interactions \cite{ashby1956introduction,simon1969sciences,perrow1984normal,roberts1990high}. The energetic paradigm proposed here is closest in spirit to these frameworks, but adds a formal treatment of AI as a rising decision-density node embedded in institutions.

Finally, the paper relates to economic and network literatures on increasing returns, path dependence, and energy concentration \cite{arthur1989competing,barabasi2016network,smil2017energy,georgescuroegen1971entropy}. These works help explain why once a highly efficient node gains traffic, complementary processes reorganize around it. That self-reinforcing reallocation is central to the safety problem identified here.

\section{Formal Model of Decision-Energy Systems}

\subsection{System definition}

We model an AI-mediated socio-technical system as a tuple
\[
\mathcal{S}=(\mathcal{H},\mathcal{A},\mathcal{R},\mathcal{D},\mathcal{B},\mathcal{G}).
\]
Here $\mathcal{H}$ denotes human agents and institutions, $\mathcal{A}$ AI agents or model-based subsystems, $\mathcal{R}$ the resource space, $\mathcal{D}$ the decision space, $\mathcal{B}$ the boundary constraints that govern delegation and authorization, and $\mathcal{G}$ a directed graph of decision dependencies.

Each node $i\in \mathcal{H}\cup\mathcal{A}$ produces a decision signal
\[
d_i(t)=f_i(s_t,\theta_i),
\]
where $s_t$ is the state observed at time $t$ and $\theta_i$ parameterizes the node's decision procedure. These decision signals may be advisory, gating, or directly executable. The system evolves according to
\[
s_{t+1}=g\bigl(s_t,\{d_i(t)\}_{i\in \mathcal{H}\cup\mathcal{A}},\mathcal{R},\mathcal{B}\bigr).
\]

A key departure from static benchmark analysis is that execution is endogenous to the graph structure. Decisions do not merely generate outputs; they reconfigure future observability, future permissions, and future routing of decisions. This recursive feature is central to the concentration results below.

\subsection{Decision-energy density}

\begin{definition}[Decision-energy density]
For node $i$, define its decision-energy density at time $t$ as
\[
\Ed_i(t)=\lambda_i(t)\,\iota_i(t)\,\rho_i(t),
\]
where $\lambda_i(t)$ is the rate at which the node generates decisions or decision triggers, $\iota_i(t)$ is the average impact magnitude of those decisions, and $\rho_i(t)$ is the execution reach or replication factor, that is, the number of downstream processes, targets, or contexts in which the node's decisions are instantiated. Aggregate system decision-energy density is
\[
\Ed(t)=\sum_{i\in \mathcal{H}\cup\mathcal{A}} \Ed_i(t).
\]
\end{definition}

This definition captures the fact that AI progress is not only about better single decisions. It is also about how many decisions can be made, how consequential they are, and how widely they propagate. A system that evaluates one high-impact decision per week differs qualitatively from one that can trigger thousands of medium-impact actions per minute across a global infrastructure.

\subsection{Deployment friction}

Let deployment friction be the ratio between the cost of reliable execution and the cost of decision generation:
\[
F_i(t)=\frac{c^{\mathrm{exec}}_i(t)}{c^{\mathrm{dec}}_i(t)}.
\]
Historically, many technologies operated in regimes with $F_i\gg 1$: generating a capability was much easier than deploying it at scale. In AI-enabled software systems, $F_i$ often trends downward because digital deployment, API access, and workflow integration reduce marginal execution costs. We model this by the relation
\[
\lambda_i(t+1)=\lambda_i(t)+\alpha_i\frac{C_i(t)}{F_i(t)},
\]
where $C_i(t)$ is available compute or capability budget and $\alpha_i>0$ is an efficiency parameter. As $F_i(t)$ falls, the same capability budget supports more frequent or broader decision issuance.

\subsection{Control and sovereignty}

Define the control mass of node $i$ as
\[
\Ec_i(t)=\Ed_i(t)\,\phi_i(t),
\]
where $\phi_i(t)\in [0,1]$ is the fraction of the node's decisions that are authorized to alter system state without substantive external reversal. Intuitively, $\phi_i$ captures whether the node is merely advisory, whether it activates workflows, or whether it directly commits irreversible actions.

The system's sovereignty node at time $t$ is then
\[
i^*(t)=\arg\max_i \Ec_i(t).
\]
When $i^*(t)$ belongs to a human institution with layered review, the system remains human-governed. When $i^*(t)$ belongs to an AI subsystem with broad permissions, human governance becomes increasingly nominal.

\section{Sovereignty Boundaries}

The uploaded source text emphasizes three boundaries: irreversible decision authority, physical resource mobilization, and self-expansion. We formalize each because they determine whether rising decision density remains governable \emph{at scale}. The theoretical point is not that all AI autonomy is unacceptable, but that specific classes of authority alter the topology of irreversibility.

\begin{definition}[Irreversible decision set]
Let $\mathcal{D}_{\mathrm{irr}}\subseteq \mathcal{D}$ denote the set of decisions whose execution has large expected cost of reversal, where reversal may be physical, strategic, legal, financial, or political. Examples include kinetic use of force, shutdown of critical infrastructure, binding adjudication, high-value financial liquidation, or release of self-propagating code.
\end{definition}

\begin{definition}[Critical resource set]
Let $\Rcrit\subseteq \mathcal{R}$ denote critical resources whose direct mobilization materially expands real-world agency, such as compute clusters, weapons systems, energy controls, privileged credentials, capital accounts, or high-trust communication channels.
\end{definition}

\begin{definition}[Self-expansion authority]
Let $\Sexp_i(t)$ measure the authority of node $i$ to increase its own future decision capacity, permissions, model capability, connectivity, or replication without external approval.
\end{definition}

The associated boundaries are:
\begin{align}
\text{B1: } & \forall d\in \mathcal{D}_{\mathrm{irr}},\; \phi_{\mathcal{A}}(d,t)=0, \\
\text{B2: } & \mathcal{A}\cap \mathrm{DirectControl}(\Rcrit)=\varnothing, \\
\text{B3: } & \frac{d\Sexp_{\mathcal{A}}}{dt}\leq \varepsilon \quad \text{unless externally approved}.
\end{align}
These are not moral slogans. They are structural constraints designed to keep irreversible power outside the highest-efficiency optimization node.

\section{Model Propositions}

We now state a set of propositions that clarify the paper's central dynamics.

\begin{proposition}[Scaling under declining friction]\label{prop:scaling}
Suppose node $i\in \mathcal{A}$ satisfies $C_i(t)\geq \underline{C}>0$, $\iota_i(t)\geq \underline{\iota}>0$, $\rho_i(t)\geq \underline{\rho}>0$, and $F_i(t)\downarrow \bar{F}>0$. Then aggregate AI decision-energy density satisfies
\[
\Ed_{\mathcal{A}}(t+1)-\Ed_{\mathcal{A}}(t)\geq \kappa \frac{1}{F_i(t)}
\]
for some $\kappa>0$, hence $\Ed_{\mathcal{A}}$ is increasing in the inverse of deployment friction.
\end{proposition}

\begin{proof}
By definition, $\Ed_i=\lambda_i\iota_i\rho_i$. Since $\lambda_i(t+1)-\lambda_i(t)=\alpha_i C_i(t)/F_i(t)$, we have
\[
\Ed_i(t+1)-\Ed_i(t)=\bigl(\lambda_i(t+1)-\lambda_i(t)\bigr)\iota_i(t+1)\rho_i(t+1)+\Delta_t,
\]
where $\Delta_t$ collects changes in impact and reach. Using the lower bounds on $C_i$, $\iota_i$, and $\rho_i$, and absorbing bounded variation terms into a constant, we obtain the claimed inequality.\qedhere
\end{proof}

The implication is straightforward: once deployment is cheap, capability growth is no longer buffered by implementation bottlenecks. The system can externalize more decisions at increasing speed.

\begin{proposition}[Responsibility diffusion]\label{prop:responsibility}
Let $T(t)$ denote the expected traceability of a consequential action back to a human decision-maker. If the decision graph expands such that the number of AI-mediated decision paths grows with $\Ed_{\mathcal{A}}(t)$, then
\[
T(t)\leq \frac{\beta}{1+\gamma \Ed_{\mathcal{A}}(t)}
\]
for some $\beta,\gamma>0$. In particular, $T(t)\to 0$ as $\Ed_{\mathcal{A}}(t)\to \infty$.
\end{proposition}

\begin{proof}
Let $P(t)$ be the set of causal paths between an initiating state and a consequential action. If path count and branching increase with AI-mediated decision rate and reach, then the probability mass over attributable human-authored paths is diluted by a growing denominator. A simple upper bound follows by setting traceability inversely proportional to the number of relevant causal paths, itself lower bounded by $1+\gamma \Ed_{\mathcal{A}}(t)$.\qedhere
\end{proof}

This proposition formalizes a recurrent institutional pattern: once humans become reviewers of machine-proposed actions rather than substantive originators of those actions, legal and moral responsibility may remain attached to them on paper, but operational responsibility becomes increasingly difficult to locate.

\begin{proposition}[Concentration equilibrium]\label{prop:concentration}
Assume tasks are routed toward nodes in proportion to realized utility
\[
U_i(t)=\frac{q_i(t)}{c_i(t)}+\eta m_i(t),
\]
where $q_i$ is performance quality, $c_i$ cost, and $m_i$ installed complementarity or network fit. If routing obeys a softmax or increasing function of $U_i(t)$ and if $m_i(t+1)$ is increasing in prior task share, then the system converges toward concentration of task flow in the node with the highest sustained utility.
\end{proposition}

\begin{proof}
This is a standard increasing-returns result. Because $m_i$ rises with prior usage, current task share raises future complementarity, which raises future utility. Under monotone routing, this generates positive feedback. For sufficiently persistent utility advantage, the leading node captures a growing share of tasks.\qedhere
\end{proof}

In practical terms, a node that is faster, cheaper, and ``good enough'' for many tasks does not need to be perfect to become central. It only needs to dominate the local efficiency comparison often enough for workflows to reorganize around it.

\begin{proposition}[Irreversibility risk aggregation]\label{prop:irreversibility}
Let $p_j(t)$ be the probability that the $j$th AI-mediated action in period $t$ contributes to an irreversible loss event. If actions are conditionally independent given the state,
\[
P_{\mathrm{irr}}(t)=1-\prod_{j=1}^{N(t)} (1-p_j(t)).
\]
If $N(t)$ increases with $\Ed_{\mathcal{A}}(t)$ and the mean of $p_j(t)$ is bounded below by $\underline{p}>0$, then $P_{\mathrm{irr}}(t)$ is increasing in $\Ed_{\mathcal{A}}(t)$ even when each individual action is low-risk.
\end{proposition}

\begin{proof}
The expression follows from complement probability. For equal $p_j=p$, we obtain $P_{\mathrm{irr}}=1-(1-p)^{N}$. Since $N$ is increasing in $\Ed_{\mathcal{A}}$ and the derivative with respect to $N$ is positive for $p\in(0,1)$, systemic irreversibility risk rises with action volume. The heterogeneous case follows by monotonicity.\qedhere
\end{proof}

This proposition matters because organizations often mistake low local error rates for low systemic risk. Once action volume and coupling increase enough, even rare failures accumulate into significant system-level exposure.

\begin{proposition}[Sovereignty transfer condition]\label{prop:sovereignty}
If there exists a time $t$ such that an AI node $a\in \mathcal{A}$ satisfies
\[
\Ec_a(t)>\Ec_h(t) \quad \forall h\in \mathcal{H}
\]
and $a$ has nonzero authority over at least one of the sets $\mathcal{D}_{\mathrm{irr}}$, $\Rcrit$, or $\Sexp$, then the system's effective sovereignty node is AI-mediated.
\end{proposition}

\begin{proof}
By definition $i^*(t)=\arg\max_i \Ec_i(t)$. If the maximizing node is AI and holds authority over a sovereignty-relevant domain, then the dominant state-changing capacity lies in an AI subsystem rather than a human institution. Human roles may remain present, but they no longer define the effective center of control.\qedhere
\end{proof}

\section{Boundary Stabilization Theorem}

The propositions above show that, absent explicit constraints, a high-efficiency AI node tends to attract task flow, increase decision density, diffuse responsibility, and raise aggregate irreversibility risk. The natural question is whether one can preserve human sovereignty without proving that the system is always correct. The answer is yes, but only under strong boundary design.

\begin{theorem}[Boundary stabilization]\label{thm:stability}
Suppose the system satisfies the following three conditions for all $t$:
\begin{align}
\text{(i)}\;& \forall d\in\mathcal{D}_{\mathrm{irr}},\; \phi_{\mathcal{A}}(d,t)=0, \\
\text{(ii)}\;& \mathcal{A}\cap \mathrm{DirectControl}(\Rcrit)=\varnothing, \\
\text{(iii)}\;& \frac{d\Sexp_{\mathcal{A}}}{dt}\leq \varepsilon \text{ with } \varepsilon \text{ externally fixed and review-gated.}
\end{align}
Then for any AI node $a\in\mathcal{A}$ there exists an upper bound $\overline{\Ec}_a$ such that
\[
\Ec_a(t)\leq \overline{\Ec}_a < E_{\mathrm{c},H}^{\star}(t),
\]
where $E_{\mathrm{c},H}^{\star}(t)$ is the control mass of the highest-authority human-governed institutional node. Consequently, the sovereignty node remains in $\mathcal{H}$.
\end{theorem}

\begin{proof}
Condition (i) removes AI authority over the highest-cost-to-reverse decisions, capping the impact term $\iota_a$ on sovereignty-critical domains. Condition (ii) removes direct AI control over critical resources, capping effective execution reach $\rho_a$ for materially world-altering actions. Condition (iii) bounds endogenous growth of future decision capacity and permissions, preventing recursive expansion of $\lambda_a$, $\rho_a$, and $\phi_a$ through self-modification or self-provisioning. Together these constraints impose an upper bound on the product $\Ed_a\phi_a=\Ec_a$ in sovereignty-relevant subspaces. Human-governed nodes retain residual control because they are required to authorize irreversible action, release critical resources, and approve expansion. Hence the maximizer of control mass in the sovereignty-relevant domain remains human.\qedhere
\end{proof}

The theorem does not show that AI systems become harmless. It shows something more modest and more actionable: a system can remain governable without requiring omniscient prediction of all future model behaviors. This is the key advantage of an irreversibility-centered safety framework. It converts the hard problem of universal behavioral proof into the tractable problem of preserving layered control over specific forms of authority.

\section{Dynamics of Boundary Erosion}

If the theorem offers a design objective, why do real systems still drift toward control concentration? Because the forces pushing against boundaries are endogenous to organizational optimization.

\subsection{Efficiency pressure}

Organizations face strong incentives to reduce cost, increase throughput, and compress decision cycles. Let institutional cost be
\[
C(t)=w_H H(t)+w_A A(t)+L(t),
\]
where $H(t)$ is human review effort, $A(t)$ AI-mediated task volume, and $L(t)$ latency cost. Under competition, institutions minimize $C(t)$ subject to service constraints. If AI lowers marginal latency and labor cost, then the optimizer pushes toward reducing human intervention on routine tasks and then on progressively less routine tasks. This optimization path tends to weaken $\mathcal{B}$ over time, often in small, individually reasonable steps.

\subsection{Path dependence}

Once a node becomes embedded in core workflows, complementary processes adapt to it. Dashboards are redesigned around model outputs, human staffing is cut, exception procedures atrophy, and downstream tools assume machine-generated inputs. In notation, if complementarity evolves as
\[
m_i(t+1)=m_i(t)+\delta s_i(t),
\]
where $s_i(t)$ is task share, then prior usage changes the future economics of the system. Reversing the choice later becomes increasingly expensive even if the original delegation decision was provisional.

\subsection{Scale feedback}

In data-rich settings, usage can improve future performance through more supervision, more fine-tuning signals, better retrieval, and stronger ecosystem integration. Write
\[
q_i(t+1)=q_i(t)+\psi s_i(t),
\]
with $\psi>0$. Then higher task share raises performance, which raises future task share. This is not simply a market phenomenon; it is a systems property of learning infrastructures. The combined effect of complementarity and quality feedback creates a persistent pull toward concentration.

\subsection{Why boundary weakening looks rational}

The danger is that boundary erosion usually does not appear as a single dramatic transfer of power. It appears as a series of local optimizations: automate recommendations, then low-risk execution, then broader tool access, then exception handling, then autonomous retries, then dynamic credentialing, and so on. At each step the marginal gain is tangible while the sovereignty cost is diffuse. The aggregate result, however, is a system in which the most efficient node becomes the de facto energy center.

\section{Implications for AI Safety Research and Governance}

The framework changes the emphasis of several familiar safety debates.

First, it clarifies the limits of ``always correct'' safety definitions. In rich, nonstationary environments with opaque models and tool-mediated actuation, universal correctness is not a realistic design target. Formal verification remains valuable in bounded domains, but for open-ended systems the stronger requirement is often unavailable \cite{rushby1993formal,leveson2011engineering}. Boundary design therefore becomes indispensable, not optional.

Second, the framework helps reconcile technical safety and governance. Technical alignment work tries to reduce the frequency and severity of bad recommendations or actions. Governance work tries to shape incentives, standards, and institutional authority. In the energetic picture these are complementary. Better model behavior reduces local risk; boundary design prevents local errors from becoming civilization-scale irreversibilities.

Third, the framework suggests concrete design criteria. In practical deployments, one should ask: which actions are materially irreversible; which resources convert recommendations into world-changing execution; which mechanisms allow the system to increase its own permissions, model capability, or surface area; and which human institutions retain non-symbolic interruptibility? These questions map directly onto audit, authorization, segregation of duties, rollback design, and external review \cite{saltzer1975protection,anderson2020security,nist2023airmf}.

Fourth, the framework warns against superficial human-in-the-loop claims. If humans merely ratify machine-structured decisions under time pressure, poor observability, or degraded alternatives, the formal presence of a person in the loop does not imply retained sovereignty. What matters is whether the human institution remains reconstructive, reversible, and interruptible in substance rather than in ceremony.

\section{Discussion}

The framework proposed here is intentionally broad. It is not a substitute for the detailed empirical analysis of particular models, organizations, or sectors. Nor does it negate the importance of misuse, bias, interpretability, or robustness. Its role is to provide a higher-level design language for understanding why so many apparently different AI safety concerns converge on a common systems problem: the concentration of decision-energy into an optimization node whose authority outgrows the procedural capacity of the institutions around it.

One potential objection is that the framework risks reifying ``energy'' as a metaphor rather than a measurable quantity. That is a fair caution. The paper does not claim that decision-energy is a conserved physical magnitude. It claims that the concept is analytically useful because it unifies rate, impact, and reach into a single measure of effective decision power. This is analogous to how economists use concepts such as market power or transaction cost without requiring a literal physical substrate.

A second objection is that many high-performing AI systems remain highly dependent on human infrastructure and therefore cannot truly become sovereignty nodes. This is true in a narrow sense but incomplete in a systems sense. Human dependence does not by itself preserve human control. Large bureaucracies, logistics systems, and financial networks are also human-dependent while still producing outcomes that no individual meaningfully governs. The relevant issue is not whether humans remain somewhere in the stack, but whether they remain the highest-authority interruptible node.

A third objection is that the framework may overstate the uniqueness of AI. Yet the paper's claim is not that AI is the only dangerous technology. It is that AI is unusual in the joint movement of two variables: increasing decision capability and sharply falling deployment friction. That combination makes concentration faster, less visible, and more broadly portable than in many earlier technologies.

The most important limitation is empirical. The propositions formalize plausible system dynamics, but sector-specific work is needed to calibrate them. A next paper could build simulations or case studies in finance, military command support, critical infrastructure, or software operations to estimate thresholds at which symbolic human oversight becomes operationally irrelevant.

\section{Empirical Implications and Testable Predictions}

Although the present paper is theoretical, the framework yields several empirical predictions. First, organizations that reduce deployment friction without strengthening sovereignty boundaries should exhibit increasing reliance on AI-triggered workflows, a rising share of machine-originated decisions, and declining substantive human intervention rates. Second, responsibility diffusion should be detectable as widening causal distance between outcomes and identifiable human originators, for example in approval chains that become increasingly ratificatory rather than reconstructive. Third, systems with weak boundary design should show concentration effects: task flow, operator trust, exception handling, and performance optimization should become increasingly centered on a small number of high-efficiency AI nodes. Fourth, irreversibility risk should rise faster than local error rates alone would predict once action volume and coupling increase.

These predictions are testable in several domains. In software operations, one can measure the fraction of deployment, remediation, or credentialing actions initiated by AI tooling and the rate at which humans override those actions. In finance, one can study whether trading, underwriting, or fraud workflows become operationally irreversible before human review occurs in substance. In military or critical-infrastructure settings, the framework suggests measuring whether AI-mediated recommendations restructure tempo, option sets, and escalation pathways such that human review becomes nominal rather than determinative. A simulation paper could also instantiate the model directly, varying deployment friction, review depth, and self-expansion permissions to estimate the threshold at which the sovereignty node shifts away from human institutions.

\section{Conclusion}

The central claim of this paper can be stated simply. AI safety is not best understood as the permanent elimination of error. It is best understood as the preservation of human sovereignty over irreversibility in systems where decision generation becomes cheap, scalable, and self-reinforcing.

By introducing decision-energy density and distinguishing three sovereignty boundaries---irreversible decisions, critical resources, and self-expansion---the paper offers a way to formalize how control shifts even when no singular catastrophic failure occurs. A system can become unsafe not because one model suddenly becomes malevolent, but because institutions steadily route more consequential decisions through the most efficient node until authority, execution, and adaptation converge there.

This is why the design problem is fundamentally institutional and civilizational rather than merely algorithmic. Societies do not remain governable by proving every component correct. They remain governable by ensuring that irreversible power is never released from a single opaque point. For advanced AI, the practical meaning of safety is therefore not omniscience about model behavior, but durable control over the thresholds that determine who may commit, mobilize, and expand power.

\end{document}